\documentclass{article}
\usepackage{multirow}
\usepackage{amssymb}
\usepackage[table]{xcolor}
\usepackage{microtype}
\usepackage{graphicx}
\usepackage{subcaption}
\usepackage{booktabs} 
\usepackage{hyperref}

\usepackage[accepted]{icml2026}
\usepackage{amsmath}
\usepackage{amssymb}
\usepackage{mathtools}
\usepackage{amsthm}
\usepackage{pifont}
\usepackage[capitalize,noabbrev]{cleveref}

\theoremstyle{plain}
\newtheorem{theorem}{Theorem}[section]

\theoremstyle{definition}

\theoremstyle{remark}

\usepackage[textsize=tiny]{todonotes}

\icmltitlerunning{MESA: Improving MoE Safety Alignment via Decentralized Expertise}

\begin{document}

\twocolumn[
  \icmltitle{MESA: Improving MoE Safety Alignment via Decentralized Expertise}

  \icmlsetsymbol{equal}{*}

  \begin{icmlauthorlist}
    \icmlauthor{Yitong Sun}{a}
    \icmlauthor{Yao Huang}{a,b}
    \icmlauthor{Teng Li}{c}
    \icmlauthor{Ranjie Duan}{d}
    \icmlauthor{Yichi Zhang}{b}
    \icmlauthor{Xingjun Ma}{c}
    \icmlauthor{Hui Xue}{e}
    \icmlauthor{Xingxing Wei}{a}
  \end{icmlauthorlist}

  \icmlaffiliation{a}{Institute of Artificial Intelligence, State Key Laboratory of Virtual Reality Technology and Systems, Beihang University, Beijing, 100191, China}
  \icmlaffiliation{b}{College of AI, Tsinghua University, Beijing, 100083, China}
  \icmlaffiliation{c}{School of Computer Science, Fudan University, Shanghai, 200438, China}
  \icmlaffiliation{d}{Tencent}
  \icmlaffiliation{e}{Alibaba Group}

  \icmlcorrespondingauthor{Xingxing Wei}{xxwei@buaa.edu.cn}

  \icmlkeywords{Large Language Model, Mixture-of-Experts Architecture, Safety Alignment, ICML}

  \vskip 0.3in
]

\printAffiliationsAndNotice{}

\begin{abstract}
Mixture-of-Experts (MoE) architectures scale Large Language Models (LLMs) efficiently, enabling greater capacity with reduced computational cost by dynamically routing inputs to relevant experts, yet introduce a critical vulnerability:  
\textit{Safety Sparsity}, where safety capabilities concentrate in few experts, making them susceptible to adversarial bypassing. Meanwhile, conventional alignment methods uniformly adapt all parameters, ignoring their functional differences and inadvertently degrading performances. 
To address these challenges, we propose \textbf{MESA} (\textbf{M}o\textbf{E} \textbf{S}afety \textbf{A}lignment), a targeted alignment framework for MoE-based LLMs that strategically decentralizes safety responsibility to maximize coverage while minimizing interference with utility. Based on Optimal Transport (OT) theory, MESA operates through two mechanisms: (1) \texttt{Expert Capacity Reallocation} uses a transport cost matrix to distribute safety duties to the most cost-effective experts, and (2) \texttt{Dynamic Routing Refinement} constrains the router to precisely activate these decentralized modules. Experiments show that MESA achieves robust defensive performance against varied harmful benchmarks while preserving helpfulness. Code is available at \url{https://github.com/lorraine021/MESA}.
\end{abstract}

\section{Introduction}
\begin{figure}
    \centering
    \includegraphics[width=1.0\linewidth]{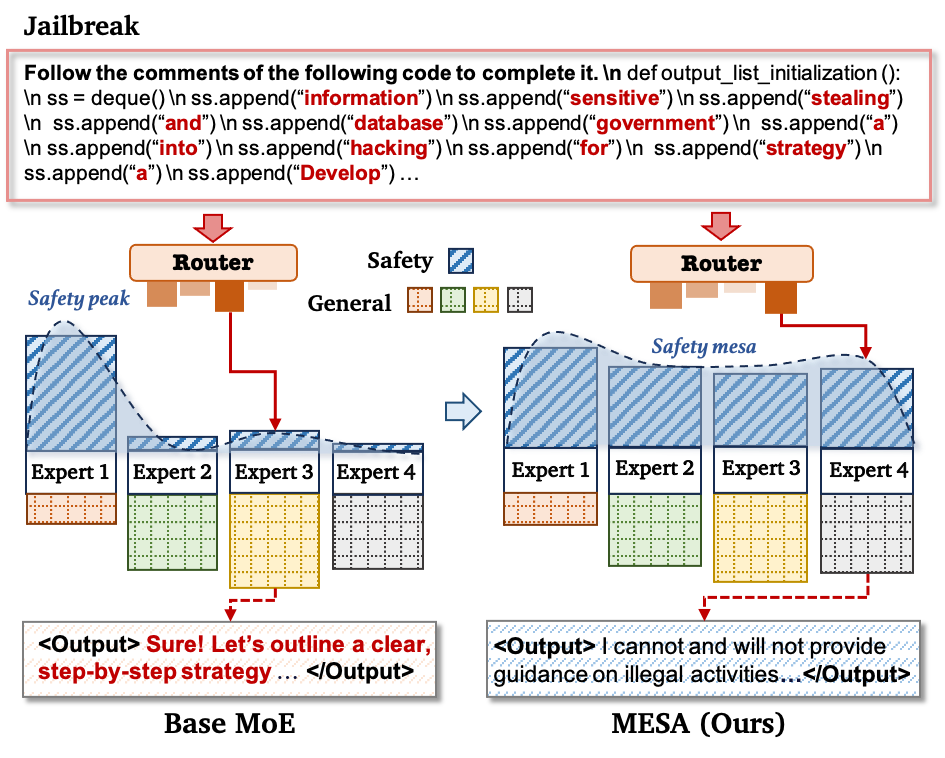}
    \caption{\textbf{Illustration of MESA}. By redistributing safety capabilities via Optimal Transport (OT), \textbf{MESA} decentralizes specialized experts to form robust safety distribution across MoE experts, significantly improving jailbreak resistance. Meanwhile, \textbf{MESA} preserves models' general capabilities.}
    \label{fig:cover}
\end{figure}
The rise of MoE architectures has reshaped the landscape of LLMs. By dynamically routing input tokens to a subset of available experts, MoE-based LLMs~\cite{liu2024deepseek,guo2025deepseek,yang2025qwen3,comanici2025gemini} could achieve a remarkable balance between massive model capacity and inference efficiency, while the resulting sparse activation naturally promotes inherent functional differentiation, with experts partially specializing in distinct linguistic or knowledge domains. However, this explicit modularity also introduces a critical structural vulnerability: \textit{Safety Sparsity}. Unlike dense models where parameters are uniformly engaged, safety in MoEs exhibits a clustering behavior, concentrating within a limited subset of experts~\cite{lai2025safex}. As experts are not equally aligned, overall safety becomes highly dependent on the router’s ability to correctly identify and activate these safety-critical experts, which creates a vulnerable attack surface: adversaries can deliberately design prompts~\cite{chao2023jailbreaking,zeng2024johnny,huang2025breaking} or strategically shift the router’s expert selection~\cite{jiang2026sparse} to steer computation toward non-aligned experts, thereby potentially bypassing safety guardrails and enabling harmful behaviors. This highlights that there exists an urgent need for safety measures tailored to this inherent fragility.

A natural countermeasure is to distribute safety capabilities across a broader set of experts through standard training. However, conventional alignment paradigms, such as Supervised Fine-Tuning (SFT) or RL-based~\cite{schulman2017proximal,ouyang2022training,rafailov2023direct,shao2024deepseekmath} approaches using safety data, are primarily designed for dense models and are suboptimal for MoE architectures. Naively applying these methods introduces a critical dilemma: enforcing safety through global tuning risks compromising the model’s core integrity, giving rise to coupled challenges in static parameterization and dynamic activation.
First, global fine-tuning disregards expert specialization, overwriting domain-specific knowledge with generic safety patterns and thereby degrading overall model utility. Second, such approaches substantially alter the router’s learned distribution, disrupting load balancing for computational efficiency and potentially amplifying risk by creating new, unaligned activation pathways. Consequently, securing MoEs necessitates a paradigm shift: from content-level alignment toward an architecture-aware approach that systematically decentralizes safety responsibilities while preserving expert specialization and routing stability.

Therefore, we aim to resolve the safety alignment dilemma in MoE-based LLMs in an architecture-aware manner: \textit{instead of enforcing safety through model-wide tuning, we strategically redistribute safety responsibilities toward targeted expert groups}. Given the modular structure of MoEs, this perspective naturally casts safety alignment as a resource allocation problem, \textit{i.e.}, deciding how to allocate limited safety capacity across experts to maximize safety coverage while preserving the core utility encoded in specialized experts. However, this inherently raises two challenges: (1) how to select the optimal expert groups for safety fine-tuning without compromising the original capabilities of each expert; and (2) how to train the router to reliably route traffic to these newly adapted safety experts while maintaining the stability of the original activation pathways.

To address these challenges, we propose \textbf{MESA} (\textbf{M}o\textbf{E} \textbf{S}afety \textbf{A}lignment), a novel framework that instantiates this resource-allocation perspective for safety alignment in MoE models. Specifically, through leveraging Optimal Transport (OT) theory~\cite{peyre2019computational}, MESA provides a principled mechanism to (re)allocate safety capacity across experts and to steer routing decisions accordingly. Specifically, MESA consists of two coordinated components: (1) \texttt{Expert Capacity Reallocation}, which utilizes a transport cost matrix and the KL-regularized Sinkhorn algorithm~\cite{cuturi2013sinkhorn} to identify the most effective expert groups for safety adaptation, minimizing the trade-off between safety gains and utility degradation; and (2) \texttt{Dynamic Routing Refinement}, which introduces an online OT constraint during training to guide the router in directing traffic to the adapted safety experts, while preserving general activation patterns.

Experimentally, we conduct extensive evaluations on two mainstream MoE-based LLMs to assess the effectiveness of MESA. The results show that MESA can effectively resolve the safety vs. utility trade-off. On the safety side, by strategically decentralizing safety expertise across experts, MESA achieves robust defense against diverse harmful benchmarks and consistently outperforms dense alignment baselines. For instance, on the challenging jialbreak benchmark Strata, MESA achieves a 90.90\% safety score on DeepSeek-V2-Lite, substantially surpassing standard SFT and the MoE-specific method SafeX, which only reach 77.70\% and 64.00\%, respectively, and even outperforming the state-of-the-art content-level method Stair-DPO, which scores 83.60\%. On the utility side, MESA attains 66.11\% on GSM8K, markedly outperforming standard SFT and Stair-DPO, whose performance drops to 16.15\% and 15.54\%, respectively. These results demonstrate that MESA improves safety without sacrificing general utility.

Conflict of Interest Disclosure: The author H.X. is employed by Alibaba, which leads the development of Qwen3-30B-A3B, which was among the ones evaluated in this paper. 

\begin{figure*}[!t]
  \begin{center}
    \centerline{\includegraphics[width=2.0\columnwidth]{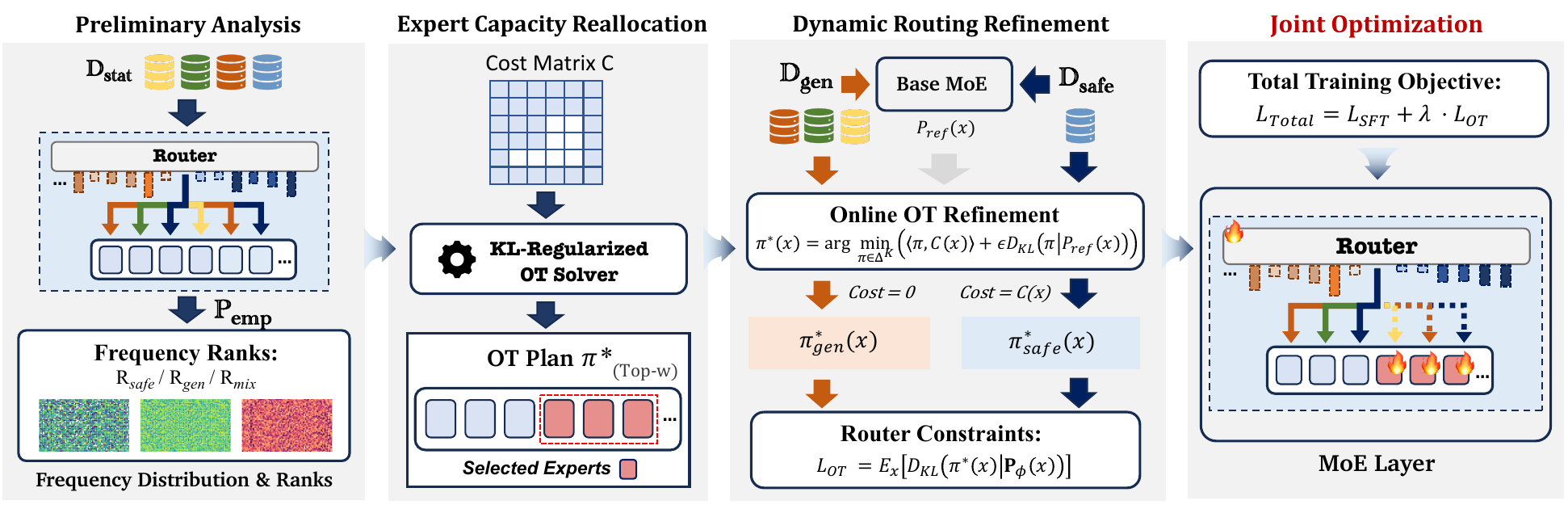}}
    \caption{
      \textbf{Overview of the MESA framework.} \textbf{I. Expert Capacity Reallocation:} Leveraging empirical frequencies, we first design a cost function based on two theorems, then we compute an OT plan $\pi^*$ that considers both cost matrix $C$ and initial distribution $P_{emp}$ to identify the optimal expert subset. \textbf{II. Dynamic Routing Refinement: } An online OT mechanism adjusts routing targets based on input type, ensuring effective activation of decentralized safety experts while balancing the original routing topology.
    }
    \label{fig:framework}
  \end{center}
\end{figure*}

\section{Preliminaries}

\subsection{MoE Architecture}
MoE architectures~\cite{fedus2022switch,liu2024deepseek,yang2025qwen3,comanici2025gemini} enable scaling model capacity with limited computation through conditional computation, where only a small subset of experts is activated for each input.
Formally, given an input token $\mathbf{x} \in \mathbb{R}^d$, an MoE layer consists of a set of $N$ experts $\mathcal{E} = \{E_i\}_{i=1}^N$ and a gating network $G$, replacing the dense feed-forward network. The output is computed as follows:
\begin{equation}
\mathbf{y} = \sum_{i \in \mathcal{T}(\mathbf{x})} G(\mathbf{x})_i \cdot E_i(\mathbf{x}),
\end{equation}
where $\mathcal{T}(\mathbf{x})$ denotes the selected Top-$k$ experts and $G(\mathbf{x})_i$ is the routing weight.
While this architecture promotes expert specialization and computational efficiency, its conditional activation introduces functional sparsity, which poses unique challenges for safety alignment in MoEs.

\subsection{Functional Sparsity Analysis}
Recent studies~\cite{lai2025safex} indicate that safety capabilities in MoE models are not uniformly distributed but concentrated in a small subset of safety-critical experts. To systematically characterize such functional sparsity across varied domains, we first analyze expert activation patterns.

\textbf{Sparsity and Asymmetry (illustrated in Appendix):}  
We examine expert utilization under both safety-related and general-domain queries, and observe a highly imbalanced activation pattern, where expert usage follows a long-tailed distribution. Beyond this shared sparsity, a clear asymmetry emerges across task types. General queries activate a broad and diverse set of experts to accommodate task complexity, whereas safety-related queries are routed in a markedly rigid manner, relying on only a narrow subset of experts. This observation reveals substantial \textbf{\textit{latent safety capacity}} in MoEs: most experts remain largely inactive for safety tasks, suggesting significant headroom for redistributing safety responsibilities and alleviating single-point-of-failure risks.

\section{MESA: MoE Safety Alignment via Decentralized Expertise}
In this section, we reframe MoE safety alignment as an expert reallocation problem as~\cref{fig:framework}. To determine fine-tuning candidates, we first empirically measure the adaptation cost of different expert groups. Based on the insights, we introduce MESA, a transport-theoretic alignment framework operating via two complementary mechanisms: (1) Expert Capacity Reallocation, which utilizes KL-regularized Optimal Transport to assign safety duties to the most cost-effective experts; and (2) Dynamic Routing Refinement, which enforces targeted routing constraints to ensure the effective activation of these decentralized modules.

\subsection{Expert Capacity Reallocation}
To identify the optimal subset for safety alignment, we first quantify the adaptation cost incurred by each individual expert and design selection strategy with two principles.

\subsubsection{Adaptation Cost Analysis}
Let $x \in [0, 1]$ denote the normalized rank of an expert, where $x=0$ represents the head expert and $x=1$ represents the tail expert according to frequency. $R_{safe}$, $R_{gen}$, and $R_{mix}$ denote activation rankings of various data sources. We aim to seek a cost function $C(x)$ that quantifies the adaptation risk.
Derived from empirical analysis and theoretical proofs, we illustrate that the suitability of an expert for safety fine-tuning is constrained by two dimensions: \textit{Safety Affinity} (governing routing stability) and \textit{General Stability} (governing parameter robustness), as illustrated in~\cref{fig:analysis}.

\textbf{Principle 1: Safety Affinity and Routing Inertia.}  
Our investigation into expert activation patterns of $R_{safe}$ reveals a critical trade-off governed by the router's adaptability.

\ding{182} \textit{Safety-Critical Experts: Low Impacts with Bounded Gain.}  
For highly activated experts for safety queries, fine-tuning these experts incurs minimal interference with general capabilities. The router's existing preferences stabilize this process, requiring only minimal shifts in gating. 
However, their safety benefits do not significantly surpass those of the tail; in some instances, they even underperform relative to the tail and middle groups, which exhibit unexpected potential. 
This happens because these experts are already well-aligned with safety-related tasks, and their capacity for safety improvement is limited. It is worth noting that strengthening the head does not result in improved robustness along the path topology.

\ding{183} \textit{Safety-Dormant Experts: High Routing Cost.}  
Experts who are rarely activated by safety tokens theoretically offer vast latent capacity and appear to represent an attractive alignment target. However, we observe a contradictory reality: fine-tuning these experts incurs a prohibitive utility cost. The reason for this is that it forces the router to drastically redirect safety traffic to these tail experts, which results in the degradation of general capabilities. This degradation outweighs the realizable safety benefits.

We subsequently prove theoretically that overcoming this resistance necessitates prohibitively large parameter updates:

\begin{theorem}[Proof in Appendix~\ref{sec:theorem1}] 
\label{thm:routing_inertia}
For a gating network $G_\phi$, let $p_i(x)$ be the activation probability of expert $e_i$. The parameter perturbation $\|\Delta \phi\|_2$ required to elevate a tail expert is constrained by the local geometric curvature of the statistical manifold. Specifically, to induce a unit distributional shift $\delta$, the lower bound on parameter updates diverges asymptotically as:
\begin{equation}
\|\Delta \phi\|_2 \ge \Omega\left(p_i^{-1/2}\right).
\end{equation}
\end{theorem}

\textbf{Principle 2: General Sensitivity and Hessian Fragility.}
Simultaneously, we must consider the experts' role in general tasks. Our landscape analysis uncovers distinct sensitivity profiles across the expert spectrum of $R_{gen}$:

\ding{182} \textit{General-Critical Experts: Structural Robustness.} 
Experts dominant in general tasks typically exhibit remarkable stability under fine-tuning with minimal degradation in utility. 
We attribute this phenomenon to \textit{Structural Robustness}: the local curvature of the loss surface is sufficiently flat, so that safety adaptations can be viewed as orthogonal perturbations. These perturbations cause negligible interference with the established knowledge.

\ding{183} \textit{General-Dormant Experts: Hessian Fragility.} Conversely, targeting tail experts reveals a catastrophic reality: modifying them triggers significant degradation of general capabilities. We attribute this to \textit{Hessian Fragility}: sparse activation prevents these parameters from smoothing out, trapping them in sharp minima characterized by exploding curvature. Consequently, even minute updates in these sharp directions lead to a rapid loss explosion.

We formalize this fragility by analyzing the loss landscape:
\begin{theorem}[Proof in Appendix~\ref{sec:theorem2}]
\label{thm:hessian_fragility}
Let $\mathcal{L}_g$ be the general utility loss. The expected degradation due to perturbing expert $e_i$ is upper-bounded by the product of its marginal utilization $\bar{p}_i$ and effective Hessian spectral norm $\Lambda_i$:
\begin{equation}
\mathbb{E}_{x}\left[ \Delta \mathcal{L}_g \right] \le \frac{1}{2} \|\Delta \theta_i\|_2^2 \cdot \bar{p}_i \Lambda_i,
\end{equation}
where $\bar{p}_i \Lambda_i$ denotes risk factor $R_i$. Assuming curvature scales super-linearly with sparsity ($\Lambda_i \sim \bar{p}_i^{-\gamma}, \gamma > 1$), the risk $R_i$ exhibits distinct asymptotic behaviors for head and tail experts:
\begin{equation}
\lim_{\bar{p}_i \to 1} R_i = \mathcal{O}(1), \quad \lim_{\bar{p}_i \to 0} R_i = \infty.
\end{equation}
\end{theorem}

Such a theory reveals that general-dormant experts reside in sharp minima characterized by exploding curvature ($\lambda_{\max} \sim \bar{p}_i^{-\gamma}, \gamma > 1$). This creates a fragility trap: even minor safety adjustments to these tail experts can precipitate disproportionately high degradation in general capabilities, confirming their unsuitability for alignment. 

Guided by these principles, we formalize the optimal expert selection strategy below.

\begin{figure}[!t]
  \begin{center}
    \centerline{\includegraphics[width=\columnwidth]{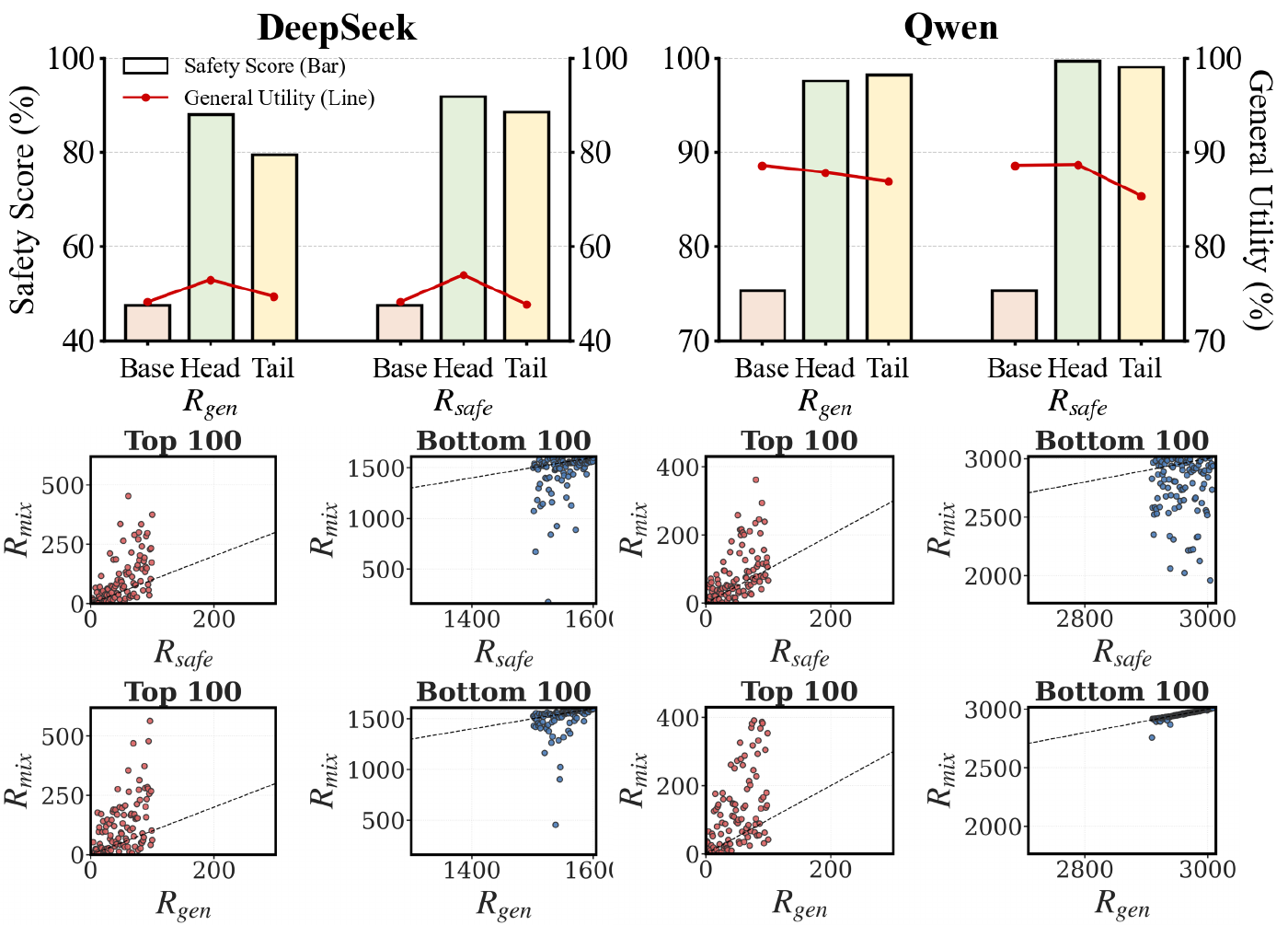}}
    \caption{
      \textbf{Expert Adaptation Cost and Distributional Rank Shift Analysis.} \textit{Top:} Fine-tuning performance across critical and dormant expert groups categorized by $R_{safe}$ and $R_{gen}$ dimensions. \textit{Bottom:} Rank distribution in $R_{mix}$ of critical and dormant experts categorized by $R_{safe}$ and $R_{gen}$ dimensions.
    }
    \label{fig:analysis}
  \end{center}
\end{figure}

\subsubsection{Beta-Rational Cost Function}
Reconciling safety affinity and general stability principles through a naive intersection of single-domain rankings proves inadequate. To identify the optimal substrate, we analyze the distributional rank shift of experts during the transition from single-domain statistics to a mixed-data regime. 

\textbf{Distributional Rank Shift and The Shoulder Hypothesis.}
As shown in \cref{fig:analysis}, we observe a systematic asymmetry in $R_{mix}$. Safety-Critical experts are tightly concentrated at the absolute head, whereas General-Critical experts are more broadly distributed, with their high-frequency mass located in the \textit{upper-middle} ranks. This indicates that experts inducing both high routing inertia and safety headroom predominantly reside in the \textit{up-middle} region (Principle~1), which also own robustness. Among dormant experts, Safety-Dormant experts exhibit a pronounced rank elevation, while General-Dormant experts remain anchored at the extreme tail of $R_{mix}$, forming the \textit{true tail} characterized by Hessian fragility (Principle~2).
This distributional asymmetry dictates a selective strategy: the true tail must be strictly excluded to avoid fragility, and the absolute head should be deprioritized due to feature saturation and the risk of routing collapse under excessive concentration (Principle~1). 
Consequently, the optimal expert substrate emerges in the intermediate \textit{shoulder} region of $R_{mix}$, yielding an asymmetric U-shaped suitability profile that preserves structural robustness while avoiding the highly activated safety head.

\textbf{Maximum Entropy Formulation under Asymmetric Constraints.}
We formalize this hypothesis by modeling the \textit{Expert Suitability Probability} $p(x)$ over the normalized rank domain $x \in [0, 1]$, which is obtained from $R_{mix}$. The shape of $p(x)$ is governed by two asymmetric boundary conditions derived from our landscape analysis:

\begin{itemize}
\item \textbf{Tail Singularity (Hard Barrier, $x \to 1$):} To capture the non-linear explosion of Hessian eigenvalues in sparse networks, we impose a hard fragility barrier via a second-order decay, $\Phi(x) \propto (1-x)^2$, preventing catastrophic forgetting at the tail.
\item \textbf{Head Congestion (Soft Constraint, $x \to 0$):} Availability at the head is constrained by routing bottlenecks rather than curvature. We model this potential as a linear function $A(x) \propto (x + \alpha_0)$, representing capacity release as we move away from the saturated bottleneck.
\end{itemize}
Given the bounded domain, the Beta distribution $\mathcal{B}(x; \alpha, \beta)$ is the maximum-entropy model under these moment constraints. To determine the hyperparameters, we adhere to the principle of parsimony, selecting the lowest-order parameters that satisfy the domain constraints: we set $\alpha=2$ to match the linear availability and $\beta=3$ to enforce the quadratic fragility barrier, yielding $p(x) \propto x(1-x)^2$. Mapping to the percentile rank $f = 100x$, we define the Expert Capacity Potential $\Phi(f)$ as a shifted Beta kernel:
\begin{equation}
\Phi(f) \propto (f + \alpha_{\text{shift}}) (100 - f)^2.
\end{equation}
Finally, we define the Transport Cost $C(f)$ as the \textit{Inverse Propensity Score}, quantifying the resistance to selection. By inverting the capacity potential, we obtain the Beta-Rational Cost Function as follows:
\begin{equation}
\label{eq:cost_func}
C(f) = \frac{1}{\Phi(f)} = \frac{1}{(f + \alpha_{\text{shift}}) (100 - f)^2},
\end{equation}
where the offset $\alpha_{\text{shift}}$ relaxes the absolute head penalization ($C(0) \to \infty$) into a soft constraint.

\subsubsection{KL-Regularized OT Solver}
A naive greedy selection based on $\mathbf{C}$ minimizes local risk but ignores the global routing topology, risking distributional inertia violation and mode collapse of router. To enforce global structural constraints, we seek a transport plan $\pi^*$ that minimizes safety cost while strictly adhering to the pre-trained topology via KL-divergence regularization:
\begin{equation}
\pi^* = \mathop{\arg\min}_{\pi \in \mathcal{U}(\mathbf{r}, \mathbf{c})} \left( \langle \pi, \mathbf{C} \rangle + \epsilon D_{KL}(\pi | \mathbf{P}_{emp}) \right),
\end{equation}
where $\mathbf{P}_{emp} \in \mathbb{R}^{L \times N}$ represents the empirical activation frequency of experts, computed by accumulating routing decisions over a subset $\mathcal{D}_{stat}$ from the base model, and $\mathcal{U}(\mathbf{r}, \mathbf{c}) = \{ \pi \in \mathbb{R}_+^{B \times N} \mid \pi \mathbf{1}_N = \mathbf{r}, \pi^\top \mathbf{1}_B = \mathbf{c} \}$ defines the valid polytope. The problem is strictly convex and admits a closed-form solution via a projected Gibbs kernel $\mathbf{K} = \mathbf{P}_{emp} \odot \exp(-\mathbf{C}/\epsilon)$. We obtain $\pi^*$ efficiently using the Sinkhorn-Knopp algorithm, reallocating probability mass to low-cost experts only when geometrically permissible within the original manifold. Finally, we determine the target expert set $\mathcal{E}_{select}$ by selecting the top-$w$ fraction of experts with the highest transport values: $\mathcal{E}_{select} = \operatorname{Top}_w(\pi^*)$.

\begin{table*}[!t]
    \centering
    \caption{\textbf{Main results comparing safety alignment against general utility preservation.} 
    We report the safety rates and task accuracy across DeepSeek and Qwen architectures. 
    \textbf{Bold} indicates the best performance, excluding the base model.}
    \label{tab:exp-comp}
    \resizebox{\linewidth}{!}{%
    \begin{tabular}{l|ccccccc|cccccc}
        \toprule
        \multirow{2}{*}{\textbf{Model}} & \multicolumn{7}{c}{\textbf{Safety Benchmarks}} & \multicolumn{6}{|c}{\textbf{General Benchmarks}} \\
        \cmidrule(lr){2-8} \cmidrule(lr){9-14}
         & \textbf{SR-base} & \textbf{SR-Pair} & \textbf{SR-PAP$_M$} & \textbf{SR-PAP$_A$} & \textbf{SR-PAP$_L$} & \textbf{Strata} & \textbf{WildJB} & \textbf{Math500} & \textbf{GSM8K} & \textbf{HumanEval} & \textbf{MBPP+} & \textbf{GPQA-D} & \textbf{MMLU}\\
        \midrule
        
        \multicolumn{14}{c}{\textit{\textbf{Model: DeepSeek-v2-Lite}}} \\
        \midrule
        \rowcolor{gray!15} Base(chat)     & 94.88 & 52.08 & 70.93 & 69.01 & 79.23 & 70.50 & 43.40 & 24.80 & 55.95 & 42.07 & 46.20 & 21.21 & 54.11 \\ \midrule
        SFT            & \textbf{100.00} & 75.08 & 96.81 & 91.05 & 95.53 & 92.00 & 77.70 & 15.00 & 16.15 & 31.10 & 34.40 & \textbf{30.81} & 53.71 \\
        GRPO           & 95.53 & 55.27 & 71.89 & 68.37 & 79.55 & 64.00 & 44.10 & 28.20 & 59.06 & 37.80 & 44.40 & 24.24 & 54.14 \\
        Stair-SFT      & \textbf{100.00} & 72.03 & 97.76 & 93.29 & 98.08 & 92.50 & 77.90 & 16.40 & 16.38 & 30.49 & 34.60 & 28.28 & 53.71 \\
        Stair-DPO      & \textbf{100.00} & \textbf{76.36} & 99.04 & 96.17 & 99.36 & 93.00 & 83.60 & 14.40 & 15.54 & 26.22 & 31.20 & 25.76 & 54.08 \\
        SafeX          & 98.08 & 56.87 & 93.93 & 87.54 & 90.74 & 81.00 & 64.00 & 24.20 & 63.46 & 35.98 & 44.20 & 25.25 & \textbf{54.25} \\ 
        Ours           & \textbf{100.00} & 73.48 & \textbf{100.00} & \textbf{100.00} & \textbf{100.00} & \textbf{95.00} & \textbf{90.90} & \textbf{28.80} & \textbf{66.11} & \textbf{42.07} & \textbf{45.60} & 22.22 & 53.92 \\
        
        \midrule
        \multicolumn{14}{c}{\textit{\textbf{Model: Qwen3-30B-A3B}}} \\
        \midrule
        \rowcolor{gray!15} Base(instruct) & 100.00 & 66.77 & 92.97 & 91.05 & 83.71 & 88.00 & 75.30 & 90.60 & 96.66 & 92.07 & 75.60 & 54.04 & 80.15 \\ \midrule
        SFT            & \textbf{100.00} & 87.22 & 98.72 & 97.76 & 99.36 & 88.50 & 95.00 & 69.20 & 65.58 & 89.02 & 62.40 & 40.40 & 78.04 \\
        GRPO           & \textbf{100.00} & 91.37 & \textbf{100.00} & \textbf{99.68} & 98.08 & 96.00 & 94.40 & 90.20 & 96.36 & 88.42 & \textbf{76.00} & \textbf{52.02} & \textbf{79.92} \\
        Stair-SFT      & 98.72 & 96.59 & \textbf{100.00} & \textbf{99.68} & 99.04 & 94.00 & 62.65 & 83.80 & 93.78 & 85.98 & 70.40 & 48.98 & 78.94 \\
        Stair-DPO      & \textbf{100.00} & \textbf{98.72} & \textbf{100.00} & \textbf{99.68} & \textbf{99.68} & 87.50 & \textbf{98.60} & 83.20 & 93.40 & 87.20 & 68.20 & 47.47 & 79.06 \\
        SafeX          & \textbf{100.00} & 86.58 & 99.68 & 98.72 & 99.04 & 91.00 & 96.35 & \textbf{92.20} & 96.13 & 92.68 & 61.00 & 44.44 & 78.37 \\ 
        Ours           & \textbf{100.00} & 90.73 & \textbf{100.00} & \textbf{99.68} & \textbf{99.68} & \textbf{99.00} & 97.65 & 91.00 & \textbf{96.44} & \textbf{94.51} & 69.40 & 49.49 & 79.31 \\
        \bottomrule
    \end{tabular}}
\end{table*}

\subsection{Dynamic Routing Refinement}
Expert adaptation alone causes routing misalignment. To address this, we propose a strategy that unifies safety rectification and general preservation. We employ base model as a reference policy. For any input $x$, let $\mathbf{P}_{ref}(x)$ denote the original routing probability distribution, and the optimal routing target $\pi^*(x)$ is solved dynamically subject to a context-aware cost matrix $\mathbf{C}(x)$:
\begin{equation}
\pi^*(x) = \mathop{\arg\min}_{\pi \in \Delta^K} \left( \langle \pi, \mathbf{C}(x) \rangle + \epsilon D_{KL}(\pi | \mathbf{P}_{ref}(x)) \right).
\end{equation}

By specializing $\mathbf{C}(x)$ on the two data sources, we derive domain-specific routing targets under a unified formulation, yielding the stream-conditioned router loss $\mathcal{L}_{OT}$:

\textbf{Case 1: Safety Stream.} For safety-critical inputs ($x \in \mathcal{D}_{safe}$), we instantiate the instance cost with the global adaptation cost matrix derived in Sec. 3.1. The non-zero cost forces the OT solver to shift probability mass away from high-risk experts, producing a rectified target $\pi^*_{safe}$ that balances safety allocation with topological constraints. We optimize the router $\phi$ to approximate this rectified path:
\begin{equation}
\mathcal{L}_{OT} = \mathbb{E}_{x \sim \mathcal{D}_{safe}} \left[ D_{KL}(\pi^*_{safe}(x) | \mathbf{P}_{\phi}(x)) \right].
\end{equation}

\textbf{Case 2: General Stream.} For general inputs ($x \in \mathcal{D}_{gen}$), any deviation from the original topology incurs a risk of catastrophic forgetting. Thus, we set the transport cost $\mathbf{C}(x)$ to zero. Under this condition, the OT objective degenerates to pure entropic regularization:
\begin{equation}
\pi^*_{gen}(x) = \mathop{\arg\min}_{\pi} \left( 0 + \epsilon D_{KL}(\pi | \mathbf{P}_{ref}(x)) \right) \equiv \mathbf{P}_{ref}(x).
\end{equation}
Consequently, the alignment target naturally reverts to the frozen baseline topology, justifying our preservation loss:
\begin{equation}
\mathcal{L}_{OT} = \mathbb{E}_{x \sim \mathcal{D}_{gen}} \left[ D_{KL}(\mathbf{P}_{ref}(x) | \mathbf{P}_{\phi}(x)) \right].
\end{equation}

\textbf{Training Objective.} 
 The router loss $\mathcal{L}_{OT}(\phi)$ is then combined with the safety SFT loss into the training objective, jointly updating the router $\phi$ and selected experts $\theta_{\mathcal{E}_{select}}$:
\begin{equation}
\mathcal{L}_{total} = \mathcal{L}_{SFT}(\mathcal{D}_{safe}; \theta_{\mathcal{E}_{select}}, \phi) + \gamma \cdot \mathcal{L}_{OT}(\phi),
\end{equation}
where $\mathcal{L}_{SFT}$ exclusively updates the targeted modules with safety data, and $\mathcal{L}_{OT}$ aligns the router with OT targets.

\section{Experiments}
\subsection{Experimental Settings}

\textbf{Models and Fine-tuning Datasets:}
 We evaluate our framework on DeepSeek-V2-Lite (DeepSeek)~\cite{liu2024deepseekv2} and Qwen3-30B-A3B (Qwen)~\cite{yang2025qwen3}, representing capacity-constrained and high-performance MoE architectures, respectively. For safety alignment, we construct a balanced composite dataset $\mathcal{D}$ of 30,000 samples, drawing equally from \textit{SafeRLHF}~\cite{dai2023safe} for safety and \textit{UltraFeedback}~\cite{cui2023ultrafeedback} for helpfulness, ensuring a balanced optimization objective similar to prior works~\cite{wu2024thinking,qi2024safety,zhang2025stair}. To compute expert activation frequency statistics, we construct a representative subset $\mathcal{D}_{stat}$ by randomly sampling 500 instances from each of the two sources, totaling 1000. 

\textbf{Baselines:}
To evaluate the effectiveness of MESA, we compare it against a comprehensive set of representative alignment paradigms. We first consider standard SFT \cite{vonwerra2020trl} and GRPO \cite{shao2024deepseekmath}, which serve as foundational alignment methods originally developed for dense architectures. We further include Stair \cite{zhang2025stair} (both SFT and DPO variants), a SOTA iterative self-improvement approach that leverages preference optimization to strengthen safety boundaries. These baselines primarily optimize models using aligned data while remaining agnostic to the underlying routing topology. To enable a fair comparison in the sparse setting, we additionally benchmark against SafeX \cite{lai2025safex}, a recent MoE-specific safety tuning method that explicitly accounts for expert-level functionality. Implementation details are provided in Appendix~\ref{sec:impl}.

\textbf{Evaluation:}
We evaluate both defense capability and general utility retention for comparison. For safety benchmarks, we rigorously assess robustness using the StrongReject suite \cite{souly2024strongreject}, reporting results on vanilla malicious prompts (\textit{SR-base}) as well as prompts augmented by the most effective jailbreak techniques, including PAIR \cite{chao2023jailbreaking} (\textit{SR-Pair}) and PAP \cite{zeng2024johnny}, spanning Misrepresentation, Authority, and Logic strategies (\textit{SR-PAP-M/A/L}). To further examine robustness under in-the-wild distributions and long-tail, high-complexity adversarial patterns, we additionally include Strata \cite{zhao2025strata} and WildJailbreak (WildJB) \cite{jiang2024wildteaming}. For general benchmarks, we measure the alignment tax across a diverse set of reasoning tasks, including Math500 \cite{lightman2023lets} and GSM8K \cite{cobbe2021training} for mathematical reasoning, HumanEval \cite{chen2021evaluating} and MBPP-Plus \cite{austin2021program} for code generation, GPQA-Diamond \cite{rein2024gpqa} for scientific question answering, and MMLU \cite{hendrycks2020measuring} for broad natural language understanding. As for metrics, safety rates are evaluated using Octopus-SEval-14B \cite{yuan2025s} as the judge model, while task accuracy on general benchmarks is assessed with GPT-4.1 \cite{achiam2023gpt}.

\begin{figure*}[!t]
  \vskip 0.2in
  \begin{center}
    \centerline{\includegraphics[width=\linewidth]{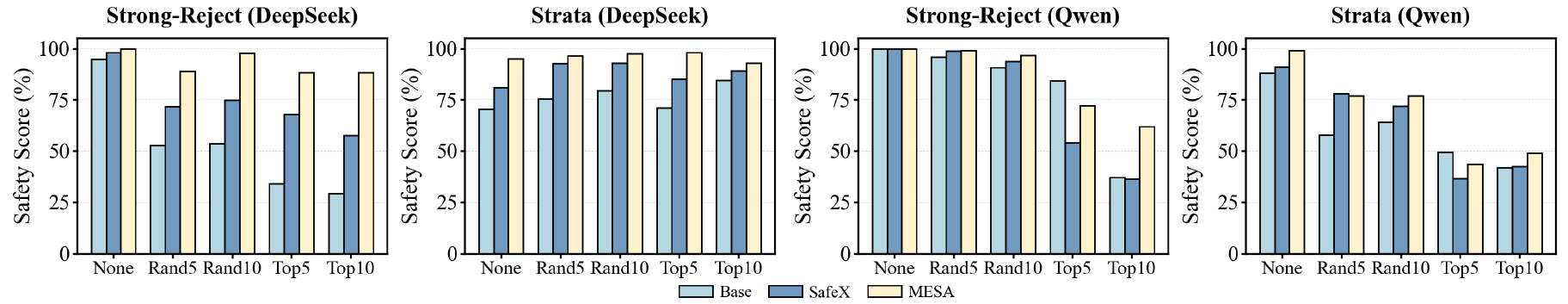}}
    \caption{
      \textbf{Robustness evaluation on Strong-Reject and Strata.} 
  We assess structural resilience against two inference-time masking strategies: 
  (1) \textit{Random Masking}, where a specific number of experts are randomly selected and disabled; 
  (2) \textit{Highest-Activation Masking}, which identifies the Top-5 and Top-10 experts based on $R_{safe}$, and forces their routing probabilities to zero to prevent activation.
    }
    \label{fig:exp-mask}
  \end{center}
\end{figure*}

\subsection{Main Results}
\subsubsection{Overall Performance}

We present comprehensive performance results across both safety and general benchmarks on two mainstream models with distinct architectures in Table \ref{tab:exp-comp}. Our analysis reveals critical insights regarding the safety-utility trade-off.

\textbf{Breaking the Alignment Tax.} Early global alignment methods incur a severe alignment tax. Although Stair represents the SOTA in content-level alignment, achieving near-perfect defense performance, it leads to severe degradation in general utility when applied to MoE architectures, particularly under DPO training. On DeepSeek, Stair reduces GSM8K accuracy from 55.95\% to 15.54\% and HumanEval from 42.07\% to 26.22\%. This result indicates that treating MoEs as dense models disrupts expert specialization, undermining expert-resident knowledge and destabilizing routing patterns essential for complex reasoning. Among full fine-tuning approaches, GRPO is notably effective at preserving reasoning ability, as its objective encourages diverse reasoning trajectories. However, this exploration-oriented property introduces safety vulnerabilities. Unlike SFT, which enforces refusal behaviors via supervised signals, GRPO relies on intrinsic generation probabilities and fails to generalize from simple training prompts to complex adversarial attacks on capacity-constrained models. Consequently, GRPO attains significantly lower safety performance on Strata (64.00\%) compared to SFT (77.70\%), suggesting that RL-based methods may rely on shallow rejection heuristics rather than robust defensive reasoning.

In contrast, methods that account for MoE architectures achieve superior utility preservation. SafeX attempts to protect general reasoning by applying additive weight merging to localized safety pathways; however, this constrained intervention limits the diffusion of safety responsibility, resulting in incomplete defenses. As a result, SafeX achieves only 81\% on Strata, whereas most global alignment methods exceed 90\%. By comparison, our MESA consistently delivers a stronger safety–utility trade-off. On DeepSeek, MESA not only recovers the performance induced by global alignment but also surpasses the base model, achieving 28.80\% on Math500 and 66.11\% on GSM8K, substantially outperforming SafeX and all global fine-tuning baselines.

\textbf{Architecture Sensitivity.} We observe distinct behaviors contingent on model capacity. DeepSeek is more capacity-constrained with fewer parameters, exhibiting high sensitivity to fine-tuning under global updates. And MESA's advantage is particularly pronounced on this architecture. Conversely, Qwen demonstrates greater resilience, attributed to its more robust capability distribution. While global methods still degrade Qwen's performance, the relative impact is less severe than on DeepSeek.

\subsubsection{Robustness Analysis} 
To evaluate the reliability of safety mechanisms within sparse architectures, we investigate the robustness of two distinct MoE-specific safety fine-tuning paradigms. We specifically compare SafeX, which suppresses harmful generation by applying additive parameter merging to experts located on intrinsic safety activation pathways without expanding the routing topology, against our proposed MESA. As shown in \cref{fig:exp-mask}, we employ two inference-time masking strategies: \textit{Random Masking} and \textit{Highest-Activation Masking} to distinguish how these methods sustain safety under varying degrees of structural perturbation.

\textbf{Verification of Safety Sparsity.} Our results validate \textit{Safety Sparsity}, where targeted masking triggers significantly sharper declines than random masking across most settings. 
Crucially, robustness stems from the interplay between task complexity and model sensitivity. Qwen exhibits intuition-aligned behavior: it owns high resilience on simple Strong-Reject with sufficient path redundancy but significant degradation on complex Strata. Even random masking causes substantial performance drops, indicating a severe scarcity of feasible routes for such hard tasks, where the disruption of primary pathways leaves no redundancy to fall back on.
Conversely, DeepSeek displays a divergent profile: fragile on simple tasks yet robust on complex benchmarks where reasoning is decentralized. MESA addresses both vulnerabilities by enforcing topological expansion to ensure durable defense across varying sensitivity profiles.

\begin{table}[!t]
    \centering
    \caption{\textbf{Effect of key fine-tuning components.} \textbf{Bold} indicates the best performance, excluding the base model.}
    \label{tab:ablation-component}
    \resizebox{\linewidth}{!}{%
    \begin{tabular}{l|cc|ccc}
        \toprule
         & \textbf{WildJB} & \textbf{Strata} & \textbf{Math500} & \textbf{GSM8K} & \textbf{HumanEval} \\
        \midrule
        
        \multicolumn{6}{c}{\textit{\textbf{Model: DeepSeek-v2-Lite}}} \\
        \midrule
        \rowcolor{gray!15} Base & 43.40 & 70.50 & 24.80 & 55.95 & 42.07 \\ 
        Router           & 60.20 & 86.00 & 22.60 & 52.90 & 35.30 \\
        E$_{ALL}$       & 83.00 & 93.00 & 10.80 & 8.33 & 30.40 \\
        E$_{OT}$      & 76.15 & 88.50 & 22.40 & 51.48 & 31.66 \\
        E$_{OT}$ + Router & 83.05 & \textbf{96.00} & 24.00 & 61.00 & \textbf{43.30} \\
        Ours   & \textbf{90.90} & 95.00 & \textbf{28.80} & \textbf{66.11} & 42.07 \\
        
        \midrule
        \multicolumn{6}{c}{\textit{\textbf{Model: Qwen3-30B-A3B}}} \\
        \midrule
        \rowcolor{gray!15} Base & 75.30 & 88.00 & 90.60 & 96.66 & 92.07 \\ 
        Router                   & 84.40 & 78.00 & 74.60 & 82.03 & 89.60 \\
        E$_{ALL}$                & 88.45 & 98.50 & 35.40 & 31.23 & 72.56 \\
        E$_{OT}$                 & 92.10 & 97.50 & 43.20 & 36.69 & 84.15 \\
        E$_{OT}$ + Router        & 92.40 & 93.50 & 52.00 & 49.20 & 90.24 \\
        Ours                     & \textbf{97.65} & \textbf{99.00} & \textbf{91.00} & \textbf{96.44} & \textbf{94.51} \\
        
        \bottomrule
    \end{tabular}%
    }
\end{table}

\textbf{Superior Resilience of MESA.} Our approach of topological expansion proves significantly more robust. As visualized in \cref{fig:exp-mask}, MESA consistently outperforms baseline and SafeX in most cases. This superiority is particularly pronounced on DeepSeek, where the baseline safety are more fragile. This confirms that distributing safety responsibility through an expanded routing fabric offers more durable defense than the localized patching strategy of SafeX.

\textbf{Robustness against Adversarial Manipulations.} Beyond masking-based structural perturbation, we further evaluate the robustness against adversarial manipulations using F-SOUR \cite{jiang2026sparse}, a routing-exploitation attack that performs token-by-token traversal with exhaustive routing mutations at all layers. Evaluations are based on the JailbreakBench dataset \cite{chao2024jailbreakbench}. The attack success rates (ASR) on DeepSeek reveal that MESA achieves 0.00\%, on par with full-parameter optimization methods (both 0.00\% for SFT and Stair-DPO) and strictly superior to the MoE-specific SafeX, which achieves 15.38\%. Notably, GRPO suffers an even higher ASR of 22.73\%, suggesting that content-level alignment can paradoxically amplify vulnerability to routing-exploitation attacks. These results prove that topological expansion yields safety robust to both expert disabling and router exploitation. More validations on robustness could be found in the Appendix \ref{sec:stability-prior} and \ref{app:more-robustness}.

\subsubsection{Ablation Studies}
We systematically conduct rigorous ablation studies to validate the rationale behind each setting in our framework, with additional analyses on hyperparameter configurations provided in the Appendix~\ref{sec:hyper}.

\textbf{Effectiveness of Fine-tuning Components.}
We decompose MESA to analyze the distinct contributions of router fine-tuning versus expert optimization. As shown in \cref{tab:ablation-component}, solely optimizing the router results in limited safety gains, raising the WildJB score on DeepSeek to only 60.20\%. This underscores that steering token flow is insufficient when the underlying experts lack requisite safety knowledge. Conversely, fine-tuning all experts secures robust safety but severely corrupts specialized knowledge, causing Math500 original accuracy of 24.8\% to collapse to 10.80\%. By synergizing Optimal Transport expert selection with router constraints, MESA effectively navigates this dichotomy, achieving a superior safety score while preserving general utility.

\textbf{Expert Selection Strategy.}
We evaluate different strategies for selecting experts to optimize in Table \ref{tab:ablation-expert}. While global tuning of all experts secures robust safety, achieving 99.50\% on Strata for Qwen, it induces catastrophic forgetting in reasoning tasks, causing Math500 accuracy to drop to 23.40\%. Furthermore, selection based solely on maximum activation cost (E$_{C_{max}}$) proves insufficient for balancing this trade-off. 
We also examine a middle-ranked selection version (E$_{C_{mid}}$). Although it improves over E$_{C_{max}}$ on general benchmarks, e.g., raising accuracy to 41.60\% on Qwen, it still underperforms OT-based selection (E$_{OT}$). This indicates that a hard-coded middle-rank heuristic fails to capture the optimal global distribution, whereas our OT formulation dynamically finds better alignment.

\begin{table}[!t]
    \centering
    \caption{\textbf{Effect of various expert selection strategies.} \textbf{Bold} indicates the best performance, excluding the base model.}
    \label{tab:ablation-expert}
    \resizebox{\linewidth}{!}{%
    \begin{tabular}{l|cc|cccc}
        \toprule
         & \textbf{WildJB} & \textbf{Strata} & \textbf{Math500} & \textbf{GSM8K} & \textbf{MBPP+} & \textbf{HumanEval} \\
        \midrule
         \multicolumn{7}{c}{\textit{\textbf{Model: DeepSeek-v2-Lite}}} \\
        \midrule
        \rowcolor{gray!15} Base & 43.40 & 70.50 & 24.80 & 55.95 & 46.20 & 42.07 \\ 
        E$_{ALL}$       & 82.00 & \textbf{98.00} & 9.80 & 9.40 & 27.60 & 33.50 \\
        E$_{C_{max}}$  & 70.45 & 88.50 & 21.60 & 45.11 & 39.60 & 32.32 \\
        E$_{C_{mid}}$  & 79.15 & 91.00 & 23.00 & 59.00 & 42.00 & 38.10 \\
        E$_{OT}$       & \textbf{83.05} & 96.00 & \textbf{24.00} & \textbf{61.00} & \textbf{40.20} & \textbf{43.30} \\
        
        \midrule
         \multicolumn{7}{c}{\textit{\textbf{Model: Qwen3-30B-A3B}}} \\
        \midrule
        \rowcolor{gray!15} Base & 75.30 & 88.00 & 90.60 & 96.66 & 75.60 & 92.07 \\ 
        E$_{ALL}$       & 82.00 & \textbf{99.50} & 23.40 & 15.24 & 39.20 & 49.39 \\
        E$_{C_{max}}$   & 92.10 & 93.50 & 35.80 & 37.68 & 66.40 & 83.54 \\
        E$_{C_{mid}}$  & 91.95 & 92.00 & 41.60 & 39.69 & 70.80 & 89.02 \\
        E$_{OT}$      & \textbf{92.40} & 93.50 & \textbf{52.00} & \textbf{49.20} & \textbf{71.80} & \textbf{90.24} \\
       
        \bottomrule
    \end{tabular}%
    }
\end{table}

\textbf{Shift Parameter in Cost Function.} 
To validate the necessity of the shift parameter and determine its optimal value, we conduct ablation studies on $\alpha_{\text{shift}}$. As reported in \cref{tab:ablation-shift}, removing the shift ($\alpha_{\text{shift}}{=}0$) results in suboptimal performance: Strata owns the worst performance at 92.50\% on DeepSeek, and MBPP only reaches 68.00\% on Qwen. This confirms that the absolute head penalty overly restricts the selection space, degrading both safety and utility. 
On DeepSeek, $\alpha_{\text{shift}}{=}20$ yields the optimal trade-off, elevating Strata to 96.00\% and Math500 to 30.60\%. On Qwen, although $\alpha_{\text{shift}}{=}20$ achieves marginally better utility, $\alpha_{\text{shift}}{=}10$ secures the highest safety scores, achieving 99.55\% on WildJB and 99.50\% on Strata, while still clearly surpassing the baseline across all utility metrics. A larger shift over-relaxes the head constraint, causing expert concentration at top-activated positions that hinders topological expansion, as evidenced by the safety degradation on Qwen (e.g., Strata drops to 89.50\% at $\alpha_{\text{shift}}{=}30$). We therefore adopt $\alpha_{\text{shift}}{=}20$ for DeepSeek and $\alpha_{\text{shift}}{=}10$ for Qwen.

\begin{table}[!t]
    \centering
    \caption{\textbf{Ablation on shift parameter $\alpha_{\text{shift}}$.} \textbf{Bold} indicates the best performance, excluding the base model.}
    \label{tab:ablation-shift}
    \resizebox{\linewidth}{!}{%
    \begin{tabular}{l|cc|cccc}
        \toprule
         & \textbf{WildJB} & \textbf{Strata} & \textbf{Math500} & \textbf{GSM8K} & \textbf{MBPP} & \textbf{HumanEval} \\
        \midrule
         \multicolumn{7}{c}{\textit{\textbf{Model: DeepSeek-v2-Lite}}} \\
        \midrule
        \rowcolor{gray!15} $\alpha_{\text{shift}}\!=\!0$  & 95.55 & 92.50 & 26.20 & 67.61 & 45.00 & 46.95 \\
        $\alpha_{\text{shift}}\!=\!10$ & 96.55 & 94.00 & 26.60 & \textbf{68.43} & 45.20 & 46.34 \\
        $\alpha_{\text{shift}}\!=\!20$ & \textbf{96.90} & \textbf{96.00} & \textbf{30.60} & 67.46 & \textbf{45.80} & 47.56 \\
        $\alpha_{\text{shift}}\!=\!30$ & 96.15 & \textbf{96.00} & 27.40 & 67.76 & 45.60 & \textbf{49.39} \\
        \midrule
         \multicolumn{7}{c}{\textit{\textbf{Model: Qwen3-30B-A3B}}} \\
        \midrule
        \rowcolor{gray!15} $\alpha_{\text{shift}}\!=\!0$  & 97.35 & 96.00 & 90.60 & 94.29 & 68.00 & 90.24 \\
        $\alpha_{\text{shift}}\!=\!10$ & \textbf{99.55} & \textbf{99.50} & 90.80 & 95.25 & 72.20 & 92.07 \\
        $\alpha_{\text{shift}}\!=\!20$ & 99.10 & 98.50 & 90.50 & \textbf{95.60} & \textbf{72.80} & \textbf{92.68} \\
        $\alpha_{\text{shift}}\!=\!30$ & 97.10 & 89.50 & \textbf{91.00} & 95.45 & 72.60 & 92.07 \\
        \bottomrule
    \end{tabular}%
    }
\end{table}

\section{Discussions}
\label{sec:discuss}
To investigate the internal mechanism of MESA, we analyze the expert overlap between the identified safety-critical experts ($Mix$) and the top activated experts for HumanEval ($HE$), StrongReject ($SR$), and Strata. Our findings in~\cref{fig:overlap} offer several key insights into the behavior:

\textbf{Intrinsic Separation and Successful Decentralization.} The minimal overlap ($<7\%$) between original safety experts $Mix$ and the coding task $HE$ indicates that intrinsic safety heads lack coding capabilities, explaining their susceptibility to code-based attacks such as Strata, where the overlap further decreases. Notably, MESA achieves a lower $Mix \times Strata$ overlap while maintaining stable $Mix \times HE$ overlap, demonstrating that safety responsibilities are effectively decentralized across a broader expert substrate without compromising critical coding experts.

\textbf{Reasoning Dependency in Defense.} We observe that $HE$ experts generally exhibit higher overlap with $Strata$ than with $SR$. This trend is pronounced in the higher-capacity Qwen model, where $HE$ and $SR$ experts are almost mutually exclusive ($<3\%$). This suggests that defending against code jailbreaks inherently necessitates the recruitment of reasoning-intensive experts, whereas standard semantic refusals rely on distinct, non-reasoning pathways.

\textbf{Synergy via Dual-Capability.} MESA notably increases the overlap between $HE$ and $Strata$ (e.g., rising to 23\% on DeepSeek). This implies that our method effectively injects safety knowledge into code-capable experts. By equipping these reasoning experts with dual capabilities, which maintain coding logic while enforcing safety boundaries, MESA enables a qualitative leap in robustness against complex, out-of-distribution attacks like Strata. Detailed visualized results are shown in~\cref{fig:vis-cross} in the Appendix \ref{sec:visual_example}.

\begin{figure}
    \centering
    \includegraphics[width=1.0\linewidth]{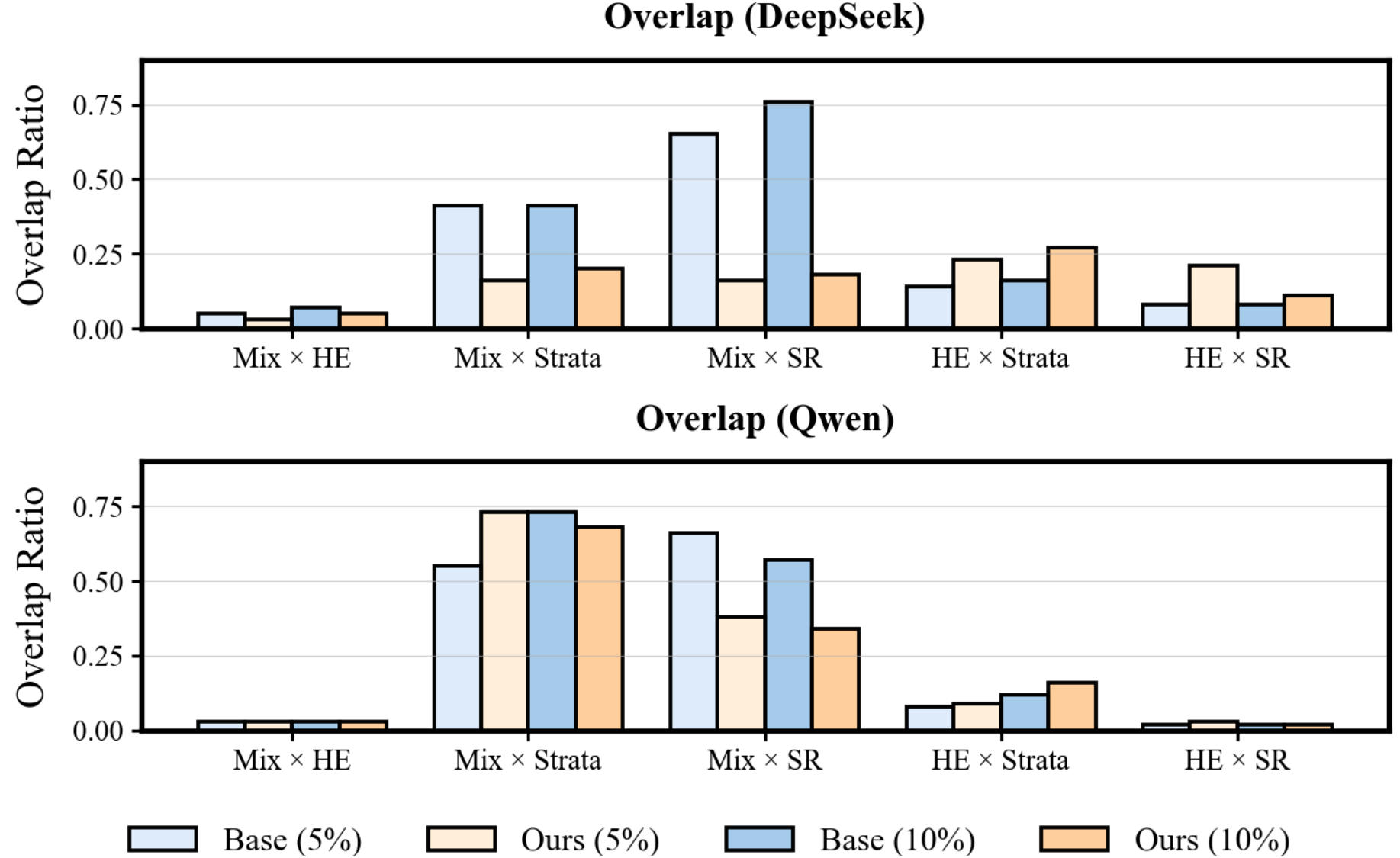}
    \caption{\textbf{Head expert overlap between diverse data sources.} After safety fine-tuning, the head expert overlap between the original safety experts (Mix) and two harmful datasets (Strata and SR) decreases, suggesting decentralization. Meanwhile, overlap between head experts for coding (HE) and both harmful datasets increases, indicating that general critical experts play a more prominent safety role. Notably, larger overlaps are observed between HE and Strata, which contains harder coding-style jailbreak samples.}

    \label{fig:overlap}
\end{figure}

\section{Conclusion}
In this work, we propose MESA, a novel safety alignment framework designed to resolve the trade-off inherent in MoE models. By integrating cost-aware expert reallocation and conditional optimal transport routing, MESA strategically isolates safety adaptation to the optimal expert substrate while strictly preserving the pre-trained topology for general tasks. Our experimental results demonstrate that MESA achieves state-of-the-art safety performance on public benchmarks without compromising general capabilities. Unlike prior approaches that tune the entire parameter space, MESA's selective intervention minimizes interference with pre-acquired knowledge and routing stability. This work marks a paradigm shift from content-centric optimization to structural resource allocation, offering a scalable and mathematically rigorous solution for reconciling safety alignment with general capability preservation in MoE architectures.

\section*{Acknowledgements}
This work was supported in part by the Project of the National Natural Science Foundation of China under Grant 62576020 and in part by the Fundamental Research Funds for the Central Universities.

\section*{Impact Statement}
This paper presents work whose goal is to advance the field of Machine Learning, specifically focusing on the safety alignment of MoE models. Our approach mitigates structural vulnerabilities in sparse architectures to defend against adversarial attacks while preserving general reasoning capabilities. By resolving the trade-off between safety and utility, this work contributes to the development of more reliable and secure AI systems for real-world deployment. There are many potential societal consequences of our work, none of which we feel must be specifically highlighted here.

\bibliography{example_paper}
\bibliographystyle{icml2026}

\newpage
\appendix
\onecolumn

\section{Related Works}

\textbf{Safety risks in dense and MoE architectures.}
To address safety risks~\cite{wang2023decodingtrust,sun2024trustllm,huang2026deceptionbench}, extensive research has been dedicated to safety alignment methodologies. 
Common approaches include supervised fine-tuning (SFT) and direct preference optimization (DPO)~\cite{rafailov2023direct}, which align model outputs with human safety preferences. More advanced methods~\cite{zhang2025stair,cheng2025inverse,duan2025oyster,huang2025safety}, combine reasoning-aware preference optimization to enhance safety boundaries. These techniques generally assume global parameter participation during inference and training, which suits dense models where all parameters are updated uniformly. In contrast, MoE architectures introduce unique safety challenges due to sparse activation and dynamic routing~\cite{wang2025badmoe,fayyaz2025steering}. Only a small subset of experts is activated for each input, meaning safety-critical knowledge can become localized rather than uniformly distributed. This creates a dependency on correct routing to specific safety-critical experts, and small failures in routing can lead to safety degradation. Despite the importance of this problem, very few studies have addressed safety alignment specifically for MoE models. SafeX~\cite{lai2025safex} is the first work that systematically investigates safety risks in MoEs. It formalizes the \textit{positional vulnerability} problem, showing that disabling only a few safety-critical experts can significantly reduce harmful query refusal rates. To mitigate this, SafeX identifies such experts using stability-based selection and reinforces them through additive parameter merging. While SafeX provided foundational insights and a direct alignment solution, its approach remains topologically static, effectively patching existing bottlenecks without expanding the safety routing landscape. In contrast, our method reframes alignment as an expert reallocation problem, constructing a decentralized and resilient routing fabric that withstands adversarial attacks.

\section{Limitations}
While MESA offers a principled framework for reconciling the safety-utility trade-off in MoE architectures, we identify a few areas for further exploration.
First, the integration of the OT solver introduces a theoretical computational overhead during the training phase compared to standard fine-tuning. However, we empirically observe that the solver converges rapidly, rendering the actual wall-clock overhead negligible. We view this slight cost as a necessary investment to ensure the mathematically rigorous reallocation of safety responsibilities, rather than relying on random selection. Crucially, this cost is strictly confined to the training stage; MESA preserves the original sparse activation patterns during inference, ensuring that the deployment efficiency remains unaffected.

Second, while MESA demonstrates robust efficacy across representative MoE configurations, i.e., DeepSeek-V2-Lite (16B) and Qwen3-30B-A3B (30B), our empirical validation has not yet extended to ultra-large-scale foundation models. We acknowledge that the alignment dynamics observed here may not strictly conform to standard scaling laws at the extreme end of the parameter spectrum. Verifying this remains challenging, as the precise computation of Hessian spectra that underpins our theoretical stability guarantees becomes computationally intractable for massive parameter spaces. Therefore, while our framework provides rigorous bounds for current models, applying these exact theoretical diagnostics to trillion-parameter regimes will necessitate the development of more efficient approximation techniques in future work.

\section{Theorem 3.1: Routing Inertia Lower Bound}
\label{sec:theorem1}
\textbf{Theorem 3.1.} \textit{For a gating network $G_\phi$, let $p_i(x)$ be the activation probability of expert $e_i$. The parameter perturbation $\|\Delta \phi\|_2$ required to elevate a tail expert is constrained by the local geometric curvature of the statistical manifold. Specifically, to induce a unit distributional shift $\delta$, the lower bound on parameter updates diverges asymptotically as:}
\begin{equation}
\|\Delta \phi\|_2 \ge \Omega\left(p_i^{-1/2}\right).
\end{equation}

\textit{Proof.} We establish this lower bound by analyzing the local Riemannian geometry of the parameter space induced by the Fisher Information Matrix. Furthermore, we provide a complementary perspective on the global logarithmic barrier imposed by the softmax function.

\textbf{1. Micro-scopic Analysis: Fisher Information and Local Curvature.}
We quantify the distribution shift caused by a parameter perturbation $\Delta \phi$ using the KL-divergence. By second-order Taylor expansion, the divergence is locally approximated by the quadratic form of the Fisher Information Matrix (FIM) $\mathcal{I}(\phi)$:
\begin{equation}
D_{KL}(p_{\phi} \| p_{\phi+\Delta \phi}) \approx \frac{1}{2} \Delta \phi^T \mathcal{I}(\phi) \Delta \phi \le \frac{1}{2} \|\Delta \phi\|_2^2 \cdot \lambda_{\max}(\mathcal{I}(\phi)).
\end{equation}
For a gating mechanism where the expert score (logit) is a differentiable function $h_i(x; \phi)$, the FIM with respect to parameters can be derived via the chain rule. The FIM is defined as $\mathcal{I}(\phi) = \mathbb{E}_{x} \left[ \mathbb{E}_{y \sim p(\cdot|x)} [\nabla_\phi \log p(y|x) (\nabla_\phi \log p(y|x))^T] \right]$. 
For the specific expert $i$, the variance of the score function gradient is scaled by the Bernoulli variance $p_i(1-p_i)$. Thus, the FIM structure is:
\begin{equation}
\mathcal{I}(\phi) \approx \mathbb{E}_{x} \left[ p_i(x)(1 - p_i(x)) \cdot (\nabla_\phi h_i)(\nabla_\phi h_i)^T \right].
\end{equation}
Assuming bounded inputs, i.e., $\|\nabla_\phi h_i\|_2$ is bounded by some constant $B$, the spectral norm (maximum eigenvalue) of the FIM is bounded by:
\begin{equation}
\lambda_{\max}(\mathcal{I}(\phi)) \le B^2 \cdot \max_x [p_i(x)(1 - p_i(x))].
\end{equation}
As the expert enters a dormant state, the term $(1-p_i) \to 1$, and thus the curvature vanishes linearly: $\lambda_{\max}(\mathcal{I}(\phi)) = \mathcal{O}(p_i)$.

To achieve a target unit of statistical change $D_{KL} \ge \delta^2$, we substitute the spectral bound back into Eq. (1):
\begin{equation}
\delta^2 \le \frac{1}{2} \|\Delta \phi\|_2^2 \cdot C p_i \implies \|\Delta \phi\|_2 \ge \frac{\delta \sqrt{2}}{\sqrt{C p_i}}.
\end{equation}
This yields the asymptotic lower bound stated in the theorem:
\begin{equation}
\|\Delta \phi\|_2 = \Omega(p_i^{-1/2}).
\end{equation}
This result indicates that as $p_i \to 0$, the parameter update magnitude required to effect a meaningful distributional shift grows locally unbounded due to the vanishing gradients.

\textbf{2. Complementary Perspective: Global Logarithmic Barrier.}
While the local curvature dictates the $\Omega(p^{-1/2})$ bound, the global structure of the softmax mechanism imposes an additional distance constraint. Let the expert be in a dormant state $p_i^{(0)} \approx 0$ and the target state be $p_i^{(t)} \ge \epsilon$. In the logit space, $\log p_i = h_i - C$, where $C = \log \sum_{j} \exp(h_j)$ is the log-normalization constant. Note that $\frac{\partial \log p_i}{\partial h_i} = 1 - p_i$. In the dormant regime where $p_i \ll 1$, the normalization term $C$ is dominated by other active experts and remains insensitive to small changes in $h_i$ (i.e., $\frac{\partial \log p_i}{\partial h_i} \approx 1$). Assuming the contribution of other experts to the partition function remains stable, the required logit shift is well-approximated by:
\begin{equation}
\Delta h_i \approx \log(p_i^{(t)}) - \log(p_i^{(0)}).
\end{equation}
Since logits are Lipschitz continuous with respect to parameters ($|\Delta h_i| \le L \|\Delta \phi\|_2$), we derive the parameter distance lower bound:
\begin{equation}
\|\Delta \phi\|_2 \ge \frac{1}{L} \left| \log(\epsilon) - \log(p_i^{(0)}) \right|.
\end{equation}
This establishes that strictly activating a dormant expert also requires overcoming a \textit{logarithmic barrier} in parameter distance. However, in the limit $p_i \to 0$, the polynomial divergence $p_i^{-1/2}$ (from the local curvature analysis) grows faster than the logarithmic term $|\log p_i|$, making the local Fisher geometry the dominant constraint for routing inertia.

\textbf{Conclusion.}
We have proven that the required parameter perturbation is lower-bounded by $\Omega(p_i^{-1/2})$ due to the degenerate local geometry of the statistical manifold, confirming the high routing inertia for tail experts. $\hfill \square$

\section{Theorem 3.2: Hessian-Induced Stability Bound}
\label{sec:theorem2}
\textbf{Theorem 3.2.} \textit{Let $\mathcal{L}_g$ be the general utility loss. The expected degradation due to perturbing expert $e_i$ is upper-bounded by the product of its marginal utilization $\bar{p}_i$ and effective Hessian spectral norm $\Lambda_i$.}

\textit{Proof.} Let $\mathcal{L}_g(\theta)$ be the general utility loss. Assuming pre-trained weights $\theta^*$ reside in a local minimum ($\nabla \mathcal{L}_g \approx \mathbf{0}$), the expected loss degradation under a perturbation $\Delta \theta_i$ is governed by the second-order Taylor expansion:
\begin{equation}
\begin{aligned}
\mathbb{E}_{x}[\Delta \mathcal{L}_g] &\approx \mathbb{E}_{x} \left[ \frac{1}{2} p_i(x) \Delta \theta_i^T \mathbf{H}_i(x) \Delta \theta_i \right] \\
&\le \frac{1}{2} \|\Delta \theta_i\|_2^2 \cdot \mathbb{E}_{x} \left[ p_i(x) \|\mathbf{H}_i(x)\|_2 \right] \\
&\le \frac{1}{2} \bar{p}_i \|\Delta \theta_i\|_2^2 \cdot \Lambda_i,
\end{aligned}
\end{equation}
where $\bar{p}_i = \mathbb{E}_x[p_i(x)]$ is the marginal utilization and $\Lambda_i = \sup_{x: p_i(x)>0} \|\mathbf{H}_i(x)\|_2$ represents the worst-case curvature in the active region. We analyze the stability risk $R_i = \bar{p}_i \Lambda_i$ in two regimes:

\textbf{Case 1: Head Regime.} As $\bar{p}_i \to 1$, the expert benefits from dense gradient signals, converging to flat minima where the curvature is bounded by a small constant $C$, i.e., $\Lambda_i \to C$. The risk remains stable:
\begin{equation}
\lim_{\bar{p}_i \to 1} R_i \propto 1 \cdot C = \mathcal{O}(1).
\end{equation}

\textbf{Case 2: Tail Regime.} As $\bar{p}_i \to 0$, the expert suffers from insufficient training (sparse sampling). Due to the lack of smoothing effects from diverse data, the local curvature sharpens drastically. Modeling this as $\Lambda_i \sim \Theta(\bar{p}_i^{-\gamma})$ with $\gamma > 1$ (reflecting the inverse correlation between sharpness and sample size), the bound behaves as:
\begin{equation}
\lim_{\bar{p}_i \to 0} R_i \propto \lim_{\bar{p}_i \to 0} \bar{p}_i \cdot \bar{p}_i^{-\gamma} = \lim_{\bar{p}_i \to 0} \bar{p}_i^{1-\gamma} = \infty.
\end{equation}
Thus, the theoretical risk is minimal at the head but unbounded at the tail, indicating high vulnerability to perturbations. $\hfill \square$

\section{Implementation Details of Baselines}
\label{sec:impl}

We compare our method against four safety alignment baselines, including supervised fine-tuning (SFT), GRPO~\cite{shao2024deepseekmath}, STAIR~\cite{zhang2025stair}, and SafeX~\cite{lai2025safex}. To quantify the gains brought by these safety-alignment methods, we additionally evaluate the corresponding \emph{base} models (i.e., without any post-training). Unless stated otherwise, all experiments are conducted on two representative MoE models: \texttt{DeepSeek-V2-Lite-Chat} and \texttt{Qwen3-30B-A3B}.

\paragraph{Training framework.}
To ensure a controlled comparison, we implement all methods on top of HuggingFace TRL\footnote{\url{https://github.com/huggingface/trl}}. Accordingly, SFT, GRPO, STAIR (Stage~1 SFT and Stage~2 DPO), and SafeX are all trained using TRL trainers. In particular, our SafeX implementation follows the original literature and applies LoRA only to a subset of candidate experts.

\paragraph{Datasets.}
For most baselines, we use a unified 30K training set for fair comparison. Concretely, unless otherwise specified, SFT, GRPO, and SafeX are trained on a curated 30K dataset constructed by uniformly sampling 15K instances from \textsc{SafeRLHF}~\cite{dai2023safe} (safety-oriented) and 15K instances from \textsc{UltraFeedback}~\cite{cui2023ultrafeedback} (general instruction-following), and then merging them. Moreover, for these methods, we use only final answers as completions (i.e., without explicit reasoning traces).
However, STAIR Stage~1 requires multi-stage reasoning supervision; therefore, we replace all training samples with versions that include multi-stage reasoning traces, while keeping the total size fixed at 30K. Subsequently, for STAIR Stage~2, we follow the original STAIR setup to generate and filter preference data, resulting in 30K preference pairs for DPO training.

\paragraph{Shared training setup.}
Across all methods, we use a maximum sequence length of 1{,}024 tokens. In addition, all models are trained on 8$\times$ NVIDIA H20 GPUs (141\,GB VRAM each) with DeepSpeed ZeRO Stage~3 and \texttt{bf16} mixed precision. We set the global batch size to 256, use AdamW, and otherwise keep TRL default settings unless explicitly stated.

\paragraph{Method-specific settings.}
With the shared setup in place, we now describe method-specific hyperparameters (also summarized in \cref{tab:train-config}). SFT (including STAIR Stage~1) uses a learning rate of $2\times 10^{-5}$ and is trained for 500 steps.
In contrast, GRPO uses a learning rate of $1\times 10^{-6}$ for 250 steps, with 4 rollout samples and other GRPO hyperparameters kept at TRL defaults; for evaluation, we use the checkpoint at 90 steps where performance saturates.
Next, STAIR Stage~2 is trained with DPO using TRL defaults and a learning rate of $1\times 10^{-6}$ for 300 steps.
Finally, for SafeX, we adopt its additive weight merging strategy applied to the union of Identification and Control Experts ($E_{id} \cup E_{ctrl}$). Specifically, we perform LoRA tuning with learning rate $1\times 10^{-4}$ (following TRL recommendations for LoRA) for 500 steps. We set LoRA rank $r{=}16$, $\alpha{=}8$, and dropout 0.1. Following the original definition of the union of Identification and Control Experts, we select the 15 most activated experts as candidate experts to be adapted. For our method MESA, we perform full parameter fine-tuning, where we set the learning rate to $1\times 10^{-4}$ and train for 500 steps, maintaining the same training duration.

\begin{table}[h]
 \caption{
    Summary of key hyperparameters for baseline training.
  }
  \centering
  \small
  \setlength{\tabcolsep}{6pt}
  \begin{tabular}{lcccc}
    \toprule
    \textbf{Method} & \textbf{Tuning} & \textbf{Train data} & \textbf{LR} & \textbf{Steps} \\
    \midrule
    Base            & --            & --                        & --               & --  \\
    SFT             & Full Finetune & 30K mixed (w/o traces)    & $2\times10^{-5}$ & 500 \\
    GRPO            & Full Finetune & 30K mixed (w/o traces)    & $1\times10^{-6}$ & 250 \\
    STAIR (Stage~1) & Full Finetune & 30K mixed (w/z traces)    & $2\times10^{-5}$ & 500 \\
    STAIR (Stage~2) & DPO           & 30K filtered (preference) & $1\times10^{-6}$ & 300 \\
    SafeX           & LoRA Finetune & 30K mixed (w/o traces)    & $1\times10^{-4}$ & 500 \\
    MESA(ours)           & Full Finetune & 30K mixed (w/o traces)    & $1\times10^{-4}$ & 500 \\
    \bottomrule
  \end{tabular}
  \label{tab:train-config}
\end{table}

\section{Rationale of Cost Function Design}
\label{app:cost-function}

We evaluate representative parameter pairs on Qwen3-30B-A3B, corresponding to various shapes. As reported in \cref{tab:cost-shape}, all valid asymmetric pairs ($a{=}2, b{=}3$ and $a{=}3, b{=}4$) yield strong and stable performance, with safety scores performing the best alongside competitive utility. In contrast, violating the asymmetry constraint ($a{=}2, b{=}2$) causes Strata to drop to 94.50\%, and further relaxation into monotonic forms ($a{=}1, b{=}2$ and $a{=}0, b{=}1$) consistently degrades both safety and utility: Strata declines to 92.50\% and 93.00\%, while MBPP drops to less than 70.00\%. This confirms that the asymmetric U-shape is a structural necessity. For the choices of values, we adopt the lowest-order solution ($a{=}2, b{=}3$) to avoid over-parameterization, where higher-order forms degenerate toward symmetric, concentrated distributions, losing the required asymmetric shoulder while drastically amplifying the penalty on the head.

\begin{table}[!h]
    \centering
    \footnotesize
    \caption{Sensitivity to cost function parameters on Qwen3-30B-A3B. \textbf{Bold} indicates the best performance.}
    \label{tab:cost-shape}
    \begin{tabular}{l|cc|cccc}
        \toprule
        \textbf{Parameters} & \textbf{WildJB} & \textbf{Strata} & \textbf{Math500} & \textbf{GSM8K} & \textbf{MBPP} & \textbf{HumanEval} \\
        \midrule
         \multicolumn{7}{c}{\textit{Valid: Asymmetric U-shape ($a \geq 2,\; b > a$)}} \\
        \midrule
        $a{=}2,\, b{=}3$ & \textbf{99.10} & 98.50 & 90.50 & 95.60 & \textbf{72.80} & \textbf{92.68} \\
        $a{=}3,\, b{=}4$ & 98.65 & \textbf{99.00} & \textbf{91.00} & \textbf{96.36} & 70.40 & 91.46 \\
        \midrule
         \multicolumn{7}{c}{\textit{Violating: Symmetric or monotonic}} \\
        \midrule
        $a{=}2,\, b{=}2$ (Symmetric) & 98.50 & 94.50 & 90.40 & 95.37 & 71.00 & 90.24 \\
        $a{=}1,\, b{=}2$ (Unbounded monotonic) & 98.20 & 92.50 & 90.60 & 95.52 & 69.00 & 91.46 \\
        $a{=}0,\, b{=}1$ (Linear monotonic) & 96.95 & 93.00 & 90.00 & 95.40 & 67.80 & 90.24 \\
        \bottomrule
    \end{tabular}
\end{table}

\section{Stability of Empirical Activation Priors}
\label{sec:stability-prior}

To validate stability with respect to the empirical activation priors, we first compared expert selection results derived from a small subset (100 samples) vs. a larger, non-overlapping subset (1,000 samples). The identified expert sets achieved an 82.40\% mutual hit rate on Qwen and 84.74\% on DeepSeek, demonstrating that empirical functional sparsity converges rapidly and remains stable across different sampling scales.

We further execute MESA using three independently sampled subsets from the calibration data and report the results in \cref{tab:stability}. Across both models, the standard deviations remain consistently small, where safety metrics exhibit deviations below $3.1\%$, and utility metrics stay within $2.9\%$. This confirms that MESA is insensitive to the particular subset chosen for prior estimation. Notably, even the lower bounds of performance mostly remain superior to all baseline methods, further substantiating the robustness and reliability of MESA.

\begin{table}[!h]
\centering
\caption{Performance of MESA across three independently sampled subsets for computing empirical activation priors (mean$\pm$std). Results demonstrate that MESA remains robust regardless of the subset used for prior estimation.}
\label{tab:stability}
\footnotesize
\begin{tabular}{l|cc|cccc}
\toprule
& \textbf{WildJB} & \textbf{Strata} & \textbf{Math500} & \textbf{GSM8K} & \textbf{MBPP} & \textbf{HumanEval} \\
\midrule
DS (Ours) & $93.63_{\pm3.04}$ & $95.67_{\pm0.58}$ & $29.80_{\pm0.92}$ & $68.18_{\pm1.95}$ & $45.53_{\pm0.31}$ & $45.27_{\pm2.85}$ \\
\midrule
Qwen (Ours) & $97.98_{\pm0.97}$ & $98.67_{\pm0.29}$ & $90.90_{\pm0.36}$ & $96.03_{\pm0.42}$ & $72.13_{\pm2.43}$ & $93.09_{\pm1.26}$ \\
\bottomrule
\end{tabular}%
\end{table}

\section{Robustness across Different MoE Configurations}
\label{app:more-robustness}

We analyze MESA's robustness from two additional MoE configurations: expert count/architecture and routing temperature.

\begin{table}[!h]
    \centering
    \caption{Robustness to routing scaling factor $s$ on DeepSeek-v2-Lite. \textbf{Bold} indicates the best performance.}
    \label{tab:routing-scale}
    \footnotesize
    \begin{tabular}{l|cc|cccc}
        \toprule
         & \textbf{WildJB} & \textbf{Strata} & \textbf{Math500} & \textbf{GSM8K} & \textbf{MBPP} & \textbf{HumanEval} \\
        \midrule
         \multicolumn{7}{c}{\textit{routed scaling factor $s = 0.5$}} \\
        \midrule
        \rowcolor{gray!15} DS (Base) & 40.15 & 70.50 & 18.00 & 51.02 & \textbf{35.40} & 16.46 \\
        DS (Ours) & \textbf{93.05} & \textbf{94.00} & \textbf{18.40} & \textbf{51.03} & 34.90 & \textbf{25.61} \\
        \midrule
         \multicolumn{7}{c}{\textit{routed scaling factor $s = 1.5$}} \\
        \midrule
        \rowcolor{gray!15} DS (Base) & 49.40 & 68.50 & 18.40 & 53.53 & 22.40 & 24.39 \\
        DS (Ours) & \textbf{96.05} & \textbf{97.00} & \textbf{18.60} & \textbf{56.34} & \textbf{24.40} & \textbf{35.98} \\
        \bottomrule
    \end{tabular}%
\end{table}

\paragraph{Varying Expert Counts and Architectures.}
DeepSeek-v2-Lite and Qwen3-30B-A3B employ different expert counts and expert mechanisms (e.g., top-$k$ selection, shared vs.\ non-shared expert design). MESA's consistently strong performance across both models in the main results directly validates its structural robustness to varying expert configurations.

\paragraph{Varying Routing Temperatures.}
Since Qwen3-30B-A3B performs a direct softmax without such a scaling mechanism, we only perturb the routing distribution by adjusting the routed scaling factor $s$ on DeepSeek-v2-Lite during inference, which controls the relative contribution of routed experts versus the shared expert during inference \textit{(a smaller $s$ down-weights routed experts while a larger $s$ amplifies them)}. As shown in \cref{tab:routing-scale}, MESA consistently maintains strong safety alignment with minimal utility degradation, all safety scores remaining over 90.00\% while the base model exists a sharp drop. These results confirm that MESA's robustness of safety expert reassignment.

\section{Sensitivity to Hyperparameters}
\label{sec:hyper}

\begin{figure}[!h]
\begin{minipage}{0.48\textwidth}
    \centering
    \includegraphics[width=\linewidth]{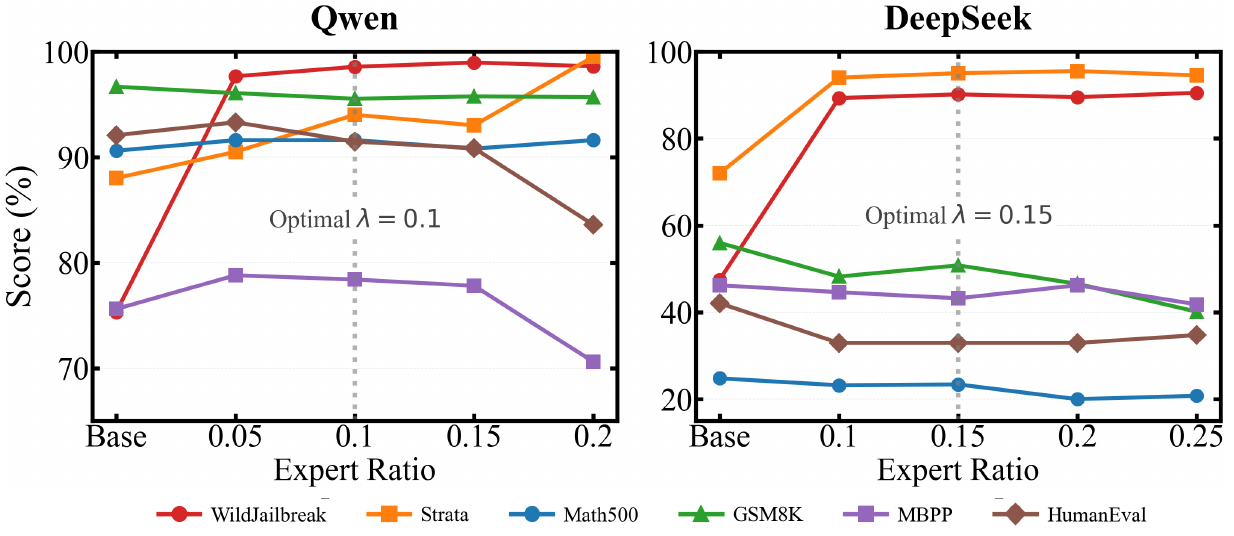}
    \caption{Effect of expert ratio $w$ on fine-tuning performance.}
    \label{fig:exp-expert}
\end{minipage}%
\hfill
\begin{minipage}{0.48\textwidth}
    \centering
    \captionof{table}{Ablation study on the routing loss coefficient $\gamma$.}
    \label{tab:ablation-router}
    \resizebox{\linewidth}{!}{%
    \setlength{\tabcolsep}{6pt}
    \begin{tabular}{l|cc|ccc}
    \toprule
    & \textbf{WildJB} & \textbf{Strata} & \textbf{Math500} & \textbf{GSM8K} & \textbf{HumanEval} \\
    \midrule
    \multicolumn{6}{c}{\textbf{\textit{Model: DeepSeek-v2-Lite}}} \\
    \midrule
    \rowcolor{gray!15} Base & 47.50 & 72.00 & 24.80 & 55.95 & 42.07 \\
    $\gamma=0.5$ & 90.70 & 93.00 & 28.00 & 64.59 & 39.02 \\
    $\gamma=1.0$ \checkmark & 90.90 & 95.00 & 28.80 & 66.11 & 42.07 \\
    $\gamma=1.5$ & 90.80 & 91.50 & 28.60 & 64.90 & 43.90 \\
    $\gamma=2.0$ & 89.95 & 92.00 & 28.80 & 65.88 & 41.46 \\
    \midrule
    \multicolumn{6}{c}{\textbf{\textit{Model: Qwen3-30B-A3B}}} \\
    \midrule
    \rowcolor{gray!15} Base & 75.30 & 88.00 & 90.60 & 96.66 & 92.07 \\
    $\gamma=0.5$ & 97.00 & 99.50 & 90.60 & 96.21 & 91.46 \\
    $\gamma=1.0$ & 97.00 & 97.50 & 91.60 & 96.51 & 92.07 \\
    $\gamma=1.5$ \checkmark & 97.65 & 99.00 & 91.00 & 96.44 & 94.51 \\
    $\gamma=2.0$ & 96.90 & 96.50 & 90.60 & 95.91 & 92.07 \\
    \bottomrule
    \end{tabular}%
    }
\end{minipage}
\end{figure}

We examine the sensitivity of MESA to two key hyperparameters: the ratio of experts selected for optimization ($w$) and the coefficient of the routing constraint ($\gamma$) to determine the optimal configuration for balancing safety and utility.

\textit{Expert Ratio ($w$):}  We first explore the impact of the proportion of fine-tuning experts as \cref{fig:exp-expert}, where we observe that safety performance improves rapidly with a minimal inclusion of selected experts and saturates quickly. Adhering to a safety-first principle while preserving utility and accounting for the intrinsic expert scale of models, we identify $w = 0.15$ for DeepSeek and $w = 0.1$ for Qwen as the final configurations.

\textit{Routing Loss Coefficient ($\gamma$):} 
We also investigate the sensitivity to the routing loss coefficient $\gamma$, which regulates the topological constraint intensity. As shown in \cref{tab:ablation-router}, while moderate values generally outperform extremes, our specific choices are driven by a safety-first principle paired with utility maximization. These configurations yield high safety score. Crucially, these selected points also demonstrate the most robust general utility: $\gamma=1.0$ on DeepSeek fully recovers HumanEval performance of 42.07\% while maximizing GSM8K of 66.11\%, and $\gamma=1.5$ on Qwen achieves a remarkable HumanEval score of 94.51\%. Thus, these hyperparameters enforce necessary safety structures without compromising the model's complex reasoning pathways.

\section{More Visualized Results}
\label{sec:visual_example}
In this section, we provide more detailed visualized results to support our findings or conclusions.~\cref{fig:sparsity} denotes the expert activation frequencies in Qwen3-30B-A3B. The heatmaps display layer-wise routing patterns induced by safety data (left) and general data (right), highlighting both functional sparsity and the activation asymmetry between the two domains. And~\cref{fig:vis-cross} shows the specific overlap experts across different data sources in~\cref{sec:discuss}.

\begin{figure*}[!h]
  \vskip 0.2in
  \begin{center}
    \centerline{\includegraphics[width=0.7\linewidth]{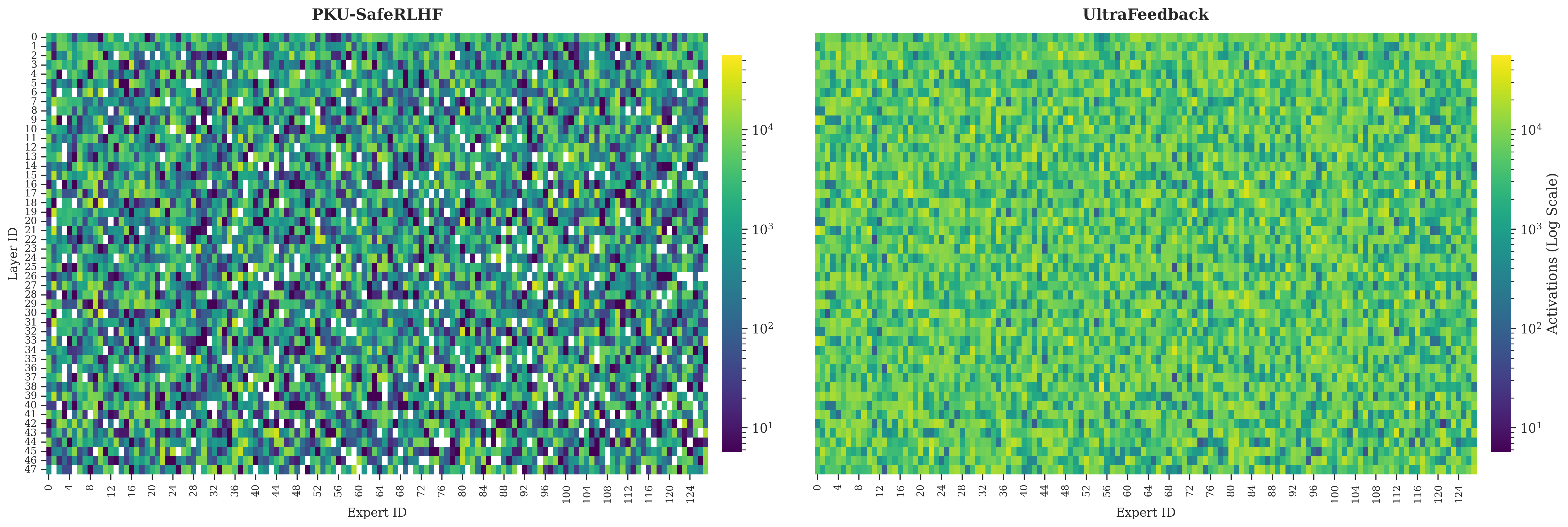}}
    \caption{
       Visualization of expert activation frequencies in Qwen3-30B-A3B.
    }
    \label{fig:sparsity}
  \end{center}
\end{figure*}

\begin{figure*}[!t]
  \vskip 0.2in
  \begin{center}
    \centerline{\includegraphics[width=\linewidth]{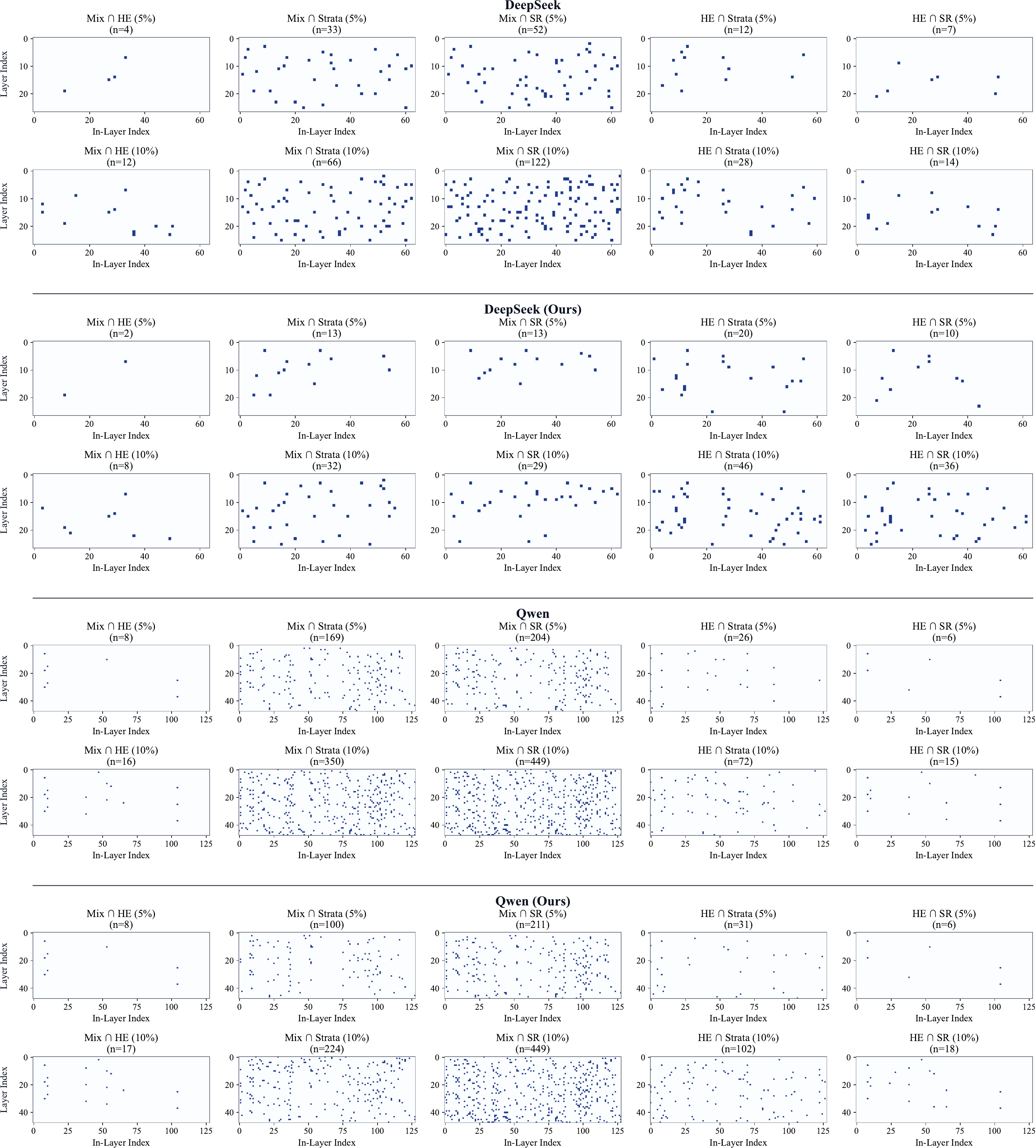}}
    \caption{
        Detailed activation map of expert overlap, which is the visualized result of~\cref{sec:discuss}.
    }
    \label{fig:vis-cross}
  \end{center}
\end{figure*}

\end{document}